\documentclass[11pt,a4paper]{article}

\usepackage{amsmath,amssymb,amsthm}
\usepackage{booktabs}
\usepackage{hyperref}
\usepackage{enumitem}
\usepackage{algorithm,algpseudocode}
\usepackage{geometry}
\usepackage{graphicx}
\usepackage{xcolor}
\usepackage{tikz}
\usepackage{caption}
\usepackage{tabularx}
\usepackage{array}
\usetikzlibrary{positioning,shapes.geometric,arrows.meta,calc}

\hypersetup{
    colorlinks=true,
    linkcolor=blue!60!black,
    citecolor=blue!60!black,
    urlcolor=blue!60!black
}

\theoremstyle{plain}
\newtheorem{theorem}{Theorem}[section]
\newtheorem{proposition}[theorem]{Proposition}

\newtheorem{corollary}[theorem]{Corollary}

\theoremstyle{definition}
\newtheorem{definition}[theorem]{Definition}

\theoremstyle{remark}
\newtheorem{remark}[theorem]{Remark}

\theoremstyle{plain}

\title{\textbf{A Theory of Least Autonomy in AI}\\[0.5em]}

\author{%
 Christophe Parisel\thanks{Email: ch.parisel@gmail.com}
}

\date{\today}

\begin{document}

\maketitle

\begin{abstract}
Least privilege, the principle that an identity should hold only the
permissions strictly required for its task, has been a foundational
primitive of access control for decades. We argue that this principle is
insufficient for agentic AI systems, which do not merely hold permissions
but can combine, approve, and amplify them across workflows and system
boundaries.

We propose \emph{least autonomy} as an appropriate generalization and
develop a formal theory. 
First, we define a compositional blast radius d(a,b) that measures structural separation between actions in an enterprise hierarchy, 
combining an ultrametric tree with lattice-valued confidentiality, integrity, and control-context labels. 
Second, we define a directed agent influence graph G(theta). An arc from U to V requires a directed shared-resource write-to-read meeting 
or a conservative undirected agent-to-agent (A2A) communication meeting, and a meeting-conditioned influence potential at or above an externally 
selected policy threshold theta. 
A catalogue-radius profile supports calibration and audit of theta. Finally, we define a collusion predicate over graph reachability that detects 
authorization composition, decision manipulation, and cross-domain capability composition.

We provide a step-by-step design procedure, a structured comparison with
classical least privilege, and an end-to-end illustrative example on a
representative enterprise hierarchy.
\end{abstract}

\section{Introduction}
\label{sec:intro}

\subsection{The Insufficiency of Least Privilege for Agentic AI}

The principle of least privilege, formalized in the
1970s~\cite{saltzer1975protection}, instructs that every principal should hold only the permissions
necessary to accomplish its designated function. For human users and
traditional software processes, this principle has provided a durable basis
for access-control design. Its application is more difficult for agentic AI
systems, whose actions may be sequenced, delegated, and composed across
workflows and system boundaries.

An agent is not merely a passive accessor of resources. It may read an
output, transform it, write to a downstream resource, trigger a workflow,
approve a privilege elevation, or hand off a task to another agent. Each
individual permission in such a chain may appear innocuous in isolation,
while the resulting configuration creates a risk not captured by a
permission-by-permission review.

Two properties of agentic systems make this particularly important:
\begin{enumerate}[label=(\roman*)]
\item \textbf{Gatekeeping.} An agent may lack direct access to sensitive
data while retaining authority to approve or enable another principal's
access. Conventional permission reviews need not capture this
control-plane influence.
\item \textbf{Composition.} An agent's operational reach may depend not
only on its own direct permissions, but also on the actions of other
agents that it can influence.
\end{enumerate}

Our framework supports both resource-mediated meetings and direct
agent-to-agent (A2A) communication meetings. For A2A communication, the
existence of a configured communication path is treated conservatively as an
undirected meeting: even an apparently one-way delegation channel may permit
reverse influence through returned content, delegated artifacts, tool
outputs, or prompt-injection payloads.

\subsection{From Permissions to Autonomy}
We propose shifting the unit of analysis from \emph{permission} to
\emph{autonomy}. Where least privilege asks, ``What can this identity
access?'', least autonomy asks, ``What authority exposure can arise when
this identity's permissions compose with those of agents it can
influence?''

Least autonomy does not replace conventional access control. Rather, it is
a complementary design criterion for identifying 
high-impact resource-mediated or communication-mediated interactions and
multi-agent capability compositions that
a permission-by-permission review may not reveal.

We formalize autonomy through a risk structure. The compositional blast
radius $d(a,b)$ measures security distance between actions in an enterprise
hierarchy. The directed influence graph then lifts this action-level
structure to agents through meeting-witnessed relationships.

Under the paper's conservative screening model, an agent's effective
autonomy includes the actions assigned to agents reachable through those
directed influence relationships. This is a capability-exposure abstraction,
not a claim that every influence path yields command authority, provenance,
or realized execution.

\subsection{Contributions}

This paper makes three contributions:
\begin{enumerate}[label=(\arabic*)]
\item \textbf{A compositional blast radius}
(Section~\ref{sec:blastradius}) that combines an ultrametric enterprise
hierarchy with a lattice of security labels.

\item \textbf{A directed agent influence graph and collusion predicate}
(Sections~\ref{sec:influencegraph}--\ref{sec:collusion}), in which arcs
require either a directed shared-resource meeting or a conservative
undirected A2A communication meeting, together with an influence potential
above an externally selected policy threshold.
The framework includes an
auditable catalogue-calibration profile, effective autonomy, and witnesses
for authorization composition, decision manipulation, and cross-domain
composition.

\item \textbf{A design procedure and comparison framework}
(Section~\ref{sec:design}) that translates the model into guidance for
security engineers.
\end{enumerate}

\newpage
\section{Background and Related Work}
\label{sec:background}

\subsection{Authorization Safety, Least Privilege, and Separation of Duty}

Classical protection models distinguish a static authorization state from the 
more difficult question of whether a principal can acquire a right through 
permitted state transitions. 
Harrison, Ruzzo, and Ullman's protection model formalized this safety question, 
showing that right acquisition is generally difficult to decide without 
restricting the protection system~\cite{harrison1976}. 
Role-based access control (RBAC) subsequently provided a practical framework 
for organizing permissions through role--permission and user--role assignments, 
while supporting constraints such as least privilege and separation of 
duty~\cite{sandhu1998}. 
Kuhn studied mutual exclusion of roles as a mechanism for enforcing separation 
of duty and characterized conditions under which such controls are safe~\cite{kuhn1997}.

Our setting differs from traditional RBAC in that principals may be autonomous 
agents and permissions may compose through cross-organizational interactions. 
Nevertheless, the distinction between static authorization, authorization evolution, 
and task-level constraints remains essential. 

\subsection{Workflow Authorization and Policy Analysis}

Authorization constraints in workflow systems address the fact that secure 
execution often depends not only on who possesses an individual permission, 
but also on which principal performs which task, in what sequence, and 
subject to what constraints. 
Bertino, Ferrari, and Atluri formalized authorization constraints for 
workflow management systems, including constraints that restrict assignments 
and executions across workflow tasks~\cite{bertino1999}. 
This literature is particularly relevant when write and read permissions are 
intended to represent a feasible business or system workflow rather than merely co-occurring permissions.

A complementary line of work concerns formal analysis of access-control policies. 
For example, Margrave supports verification and change-impact analysis for 
access-control policies~\cite{fisler2005}. 
These approaches motivate treating the proposed criterion not only as a descriptive 
score, but as an analyzable policy property with explicit assumptions about 
policy semantics, workflow feasibility, and controls.

\subsection{Graph-Based Security Analysis}

Graph-based security analysis has long represented multi-step compromise and reachability 
through attack graphs. Ammann, Wijesekera, and Kaushik developed graph-based 
vulnerability analysis to reason about paths induced by network configuration 
and exploit conditions~\cite{ammann2002}; Sheyner et al.\ similarly studied automated 
generation and analysis of attack graphs~\cite{sheyner2002}. 
These models differ from the present approach because their edges are justified by 
explicit preconditions and state transitions.

The proposed influence graph instead aggregates permission relationships using 
resource-hierarchy distance and security labels. 
It should therefore be interpreted as a static authorization-risk abstraction unless 
each retained edge is linked to a concrete feasible write-to-read workflow, 
provenance relation, or tool-mediated transition. 
Under such a relation, the metric can rank the severity of established channels.

\subsection{Lattice-Based Information-Flow Control}

Denning's foundational work~\cite{denning1976} showed that information-flow
policies can be expressed using a lattice of security classes together with
a \emph{can-flow-to} relation. Subsequent work in language-based
security~\cite{sabelfeld2003} developed these ideas in programming-language
settings, including formal accounts of information-flow enforcement.

The present model does not claim a noninterference result or an enforcement
theorem for information flow. Rather, it uses a lattice-valued label
structure to accumulate confidentiality, integrity, and gatekeeping context
along paths in an enterprise hierarchy. The connection to lattice-based flow
control is therefore conceptual: both approaches treat security properties
as compositional rather than as attributes of isolated permissions.

\subsection{Tool-Enabled Agents and Privilege Control}

Recent work shows that language-model agents become security-critical when 
they process untrusted external content while retaining authority to invoke 
consequential tools. 
ToolEmu evaluates risks in tool-mediated language-model agents~\cite{ruan2024}, 
while AgentDojo evaluates prompt-injection attacks and defenses in realistic 
agent-tool environments~\cite{agentdojo2024}. 
More recent systems explicitly pursue least-privilege enforcement for 
tool-calling agents, including Prompt Flow Integrity~\cite{kim2025pfi}, 
Progent~\cite{shi2025progent}, and MiniScope~\cite{zhu2025miniscope}.

These systems focus primarily on runtime isolation, tool-call authorization, 
or adversarial evaluation of individual agents. In contrast, this paper 
studies a static design-time question: how permissions distributed across multiple 
enterprise agents may combine across organizational scopes and security domains. 
The contribution should therefore be framed as complementary to runtime tool 
authorization: the proposed analysis can identify permission configurations that 
merit runtime mediation, approval gates, or least-privilege redesign.

\section{The Resource Hierarchy and Blast Radius}
\label{sec:blastradius}

\subsection{Enterprise Hierarchy as an Ultrametric Tree}

We model the enterprise information system as a rooted tree
$\mathcal{T} = (N, E_T)$, where leaf nodes correspond to individual
resources and internal nodes to organizational scopes. We assign depth
$D(v) \in \mathbb{N}_0$ to each node, with $D(\mathrm{root}) = 0$.

The tree induces a natural ultrametric:
For $u,v\in N$, define
\[
\operatorname{ultra}(u,u):=0,
\]
and, for $u\neq v$, define
\[
\operatorname{ultra}(u,v)
:=
2^{-\left(2D\!\bigl(\operatorname{lca}(u,v)\bigr)+1\right)}.
\]
This satisfies the strong triangle inequality and decreases by a factor of
4 for each additional depth level of the LCA, a property exploited in
Section~\ref{sec:lattice} to bound the lattice multiplier.

\subsection{The Action Permission Model}

Each action \(a\in\mathcal A\) is an admissible action instance
\[
a=(v,\mathit{CF},\mathit{RP},\mathit{SS},\mathit{Op}),
\]
where \(v\in N\) is the associated resource node,
\(\mathit{CF}\) is the control-flow classification,
\(\mathit{RP}\) is the resource-provider class,
\(\mathit{SS}\) is the sensitivity class, and
\(\mathit{Op}\) is the operation. Define
\[
\operatorname{node}(a):=v.
\]

\begin{remark}
This model is not platform-specific. It applies to cloud resources, SaaS
platforms, CI/CD systems, approval workflows, secrets stores, and custom
applications. It captures cross-organizational influence (via resource write then read) and
gatekeeping (approval authority) and resource-mediated influence.
Direct A2A communication is represented separately as an agent-pair meeting
relation rather than as a property of an individual resource action.
\end{remark}

\subsection{The Sensitivity Lattice}
\label{sec:lattice}

\begin{definition}[Local label]
For each node $v \in N$, the \emph{local label} is
$\ell(v) = (\mathit{flow},\, \mathit{axes})$ where
\[
\mathit{flow}\in\{\emptyset,GK,Ac\},
\qquad
\emptyset<GK<Ac,
\qquad
\mathit{axes}\subseteq\{C,I\},
\]
with
\[
\mathit{flow}=\emptyset
\Longrightarrow
\mathit{axes}=\emptyset.
\]
\end{definition}

The node-level coordinate \(\mathit{flow}\) is a policy-label attribute and
is distinct from the action-level classification \(\mathit{CF}(a)\).

This yields nine possible labels shown in Figure~\ref{fig:labels}.

\begin{figure}[ht]
\centering
\caption{Local label encoding lattice.}
\label{fig:labels}
\begin{tikzpicture}[
    >=Latex, scale=0.75, font=\small,
    level/.style={draw, rounded corners, align=center,
                  minimum width=2.0cm, minimum height=0.7cm}
]
\node[level] (L0)  at (0,0)    {$L_0$\\$(\emptyset,\emptyset)$};
\node[level] (L1)  at (0,-2)   {$L_1$\\$(GK,\emptyset)$};

\node[level] (L3C) at (-3,-4)  {$L_{3C}$\\$(GK,\{C\})$};
\node[level] (L2)  at (0,-4)   {$L_2$\\$(Ac,\emptyset)$};
\node[level] (L3I) at (3,-4)   {$L_{3I}$\\$(GK,\{I\})$};

\node[level] (L5C) at (-3,-6) {$L_{5C}$\\$(Ac,\{C\})$};
\node[level] (L4)  at (0,-6)  {$L_4$\\$(GK,\{C,I\})$};
\node[level] (L5I) at (3,-6)  {$L_{5I}$\\$(Ac,\{I\})$};

\node[level] (L6)  at (0,-8)  {$L_6$\\$(Ac,\{C,I\})$};

\draw[->, thick] (L0) -- (L1);

\draw[->, thick] (L1) -- (L2);
\draw[->, thick] (L1) -- (L3C);
\draw[->, thick] (L1) -- (L3I);

\draw[->, thick] (L2) -- (L5C);
\draw[->, thick] (L2) -- (L5I);

\draw[->, thick] (L3C) -- (L4);
\draw[->, thick] (L3C) -- (L5C);

\draw[->, thick] (L3I) -- (L4);
\draw[->, thick] (L3I) -- (L5I);

\draw[->, thick] (L4) -- (L6);
\draw[->, thick] (L5C) -- (L6);
\draw[->, thick] (L5I) -- (L6);
\end{tikzpicture}
\end{figure}
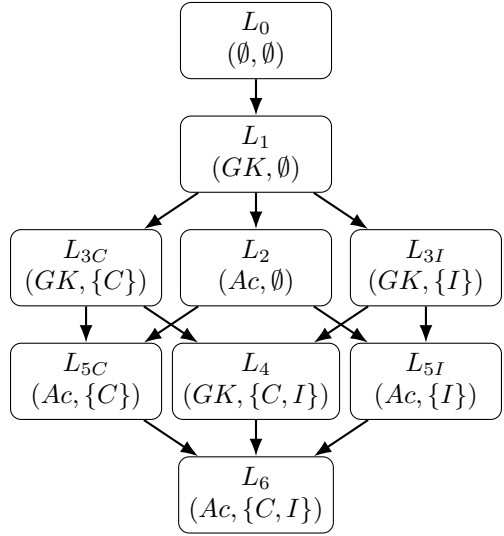

$L_{3C}$ and $L_{3I}$ are incomparable; so are $L_{5C}$ and $L_{5I}$.
The \emph{join operator} $\otimes$ is defined componentwise:
\[
  (f_1, A_1) \otimes (f_2, A_2) := \bigl(\max(f_1, f_2),\; A_1 \cup A_2\bigr).
\]
The join is commutative, associative, and idempotent.

\subsection{Path-Accumulated Labels}

\begin{definition}[Path-accumulated label]
For any node $v \in N$, let $\pi(v)$ be the path from root to $v$. The
\emph{path-accumulated label} is: $J(v) = \bigvee_{u \in \pi(v)} \ell(u)$.
\end{definition}

The join correctly detects when confidentiality and integrity risks
accumulate jointly: for example, suppose that an ancestor \(u\) of a resource node \(v\) has
local label
\[
\ell(u)=L_{3C}=(GK,\{C\}),
\]
that \(v\) has local label
\[
\ell(v)=L_{3I}=(GK,\{I\}),
\]
and that every other node on \(\pi(v)\) has label \(L_0\). Then
\[
\begin{aligned}
J(v)
&=
L_{3C}\otimes L_{3I}\\
&=
(GK,\{C\})\otimes(GK,\{I\})\\
&=
\bigl(\max\{GK,GK\},\{C\}\cup\{I\}\bigr)\\
&=
(GK,\{C,I\})\\
&=
L_4.
\end{aligned}
\]
Thus, the accumulated label records that both confidentiality and integrity
sensitivity occur along the path to \(v\), even though neither local label
alone is \(L_4\).

\subsection{From Labels to Multipliers}

We assign a multiplier $F(\ell) \in [1,2]$ to each label, monotone with
respect to the lattice order. 
Labels related by the confidentiality–integrity symmetry receive equal values 
(in particular, $L_3C$ and $L_3I$, and $L_5C$ and $L_5I$); 
other incomparable labels may receive different policy values. 
This avoids privileging confidentiality over integrity while preserving monotonicity.

\begin{table}[ht]
\centering
\caption{Lattice labels and their multipliers.}
\label{tab:multipliers}
\begin{tabular}{@{}lll@{}}
\toprule
Label & $F$ & Notes \\
\midrule
$L_0$    & 1.00 & Baseline \\
$L_1$    & 1.10 & \\
$L_2$    & 1.20 & \\
$L_{3C}$ & 1.35 & \\
$L_{3I}$ & 1.35 & Equal to $L_{3C}$ (incomparable) \\
$L_4$    & 1.50 & \\
$L_{5C}$ & 1.75 & \\
$L_{5I}$ & 1.75 & Equal to $L_{5C}$ (incomparable) \\
$L_6$    & 2.00 & Maximum \\
\bottomrule
\end{tabular}
\end{table}

\begin{remark}[$F$ range]
The ultrametric uses a base-4 gap between depth bands: each additional LCA
depth level reduces $\mathrm{ultra}$ by a factor of 4. For the multiplier
to refine risk \emph{within} a band without disrupting ordering
\emph{across} bands, we require $F_{\max}/F_{\min} < 4$. The choice
$F_{\max} = 2$, $F_{\min} = 1$ satisfies this.
\end{remark}

\subsection{The Compositional Blast Radius}

Let
\[
\operatorname{node}:\mathcal{A}\to N
\]
map each action to its associated resource node in the enterprise tree $\mathcal T$.

\begin{definition}[Compositional blast radius]
For actions \(a,b\in\mathcal A\), let
\[
u_a:=\operatorname{node}(a),
\qquad
u_b:=\operatorname{node}(b).
\]

Extend \(F\) to pairs of accumulated labels by
\[
F(\ell_1,\ell_2)
:=
\max\!\left\{
F(\ell_1),
F(\ell_2)
\right\},
\qquad
\ell_1,\ell_2\in\mathcal L.
\]
Thus,
\[
F\!\bigl(
J(u_a),
J(u_b)
\bigr)=
\max\!\left\{
F\!\bigl(J(u_a)\bigr),
F\!\bigl(J(u_b)\bigr)
\right\}.
\]

The \emph{compositional blast radius} between \(a\) and \(b\) is
\[
d_{\mathcal A}(a,b)
:=
F\!\bigl(
J(u_a),
J(u_b)
\bigr)
\cdot
\operatorname{ultra}(u_a,u_b).
\]
For brevity, write \(d(a,b):=d_{\mathcal A}(a,b)\).
\end{definition}

The distance \(d_{\mathcal A}\) measures hierarchical and sensitivity-weighted
\emph{structural span}; it is not an operation-severity metric. In particular,
actions associated with the same resource node have zero structural distance
even when their operations differ. Operation type and control-flow
classification are instead used elsewhere in the policy catalogue and
collusion analysis.

The function \(d_{\mathcal A}\) is symmetric. Since distinct actions may be
associated with the same resource node, it is a pseudoultrametric on actions:
distinct actions at the same node have blast radius zero. Its induced
function on resource nodes satisfies the ultrametric strong triangle
inequality, as proved in Appendix~\ref{app:ultrametric}.

\section{The Agent Influence Graph}
\label{sec:influencegraph}

\subsection{Meetings and Influence Potentials}

Let \(\mathcal U\) be the set of agents. Let
\(\mathcal A_{\mathrm{pol}}\) be a finite, agent-independent catalogue of
admissible action instances, each consisting of a resource node,
control-flow classification, and operation, determined by the enterprise
resource topology and policy. For each agent \(U\in\mathcal U\), assume
\[
\mathcal A(U)\subseteq\mathcal A_{\mathrm{pol}}.
\]

Define the write and read action sets of \(U\) by
\[
\mathcal W(U)
:=
\{a\in\mathcal A(U):\mathit{Op}(a)=\mathrm{Write}\},
\qquad
\mathcal R(U)
:=
\{a\in\mathcal A(U):\mathit{Op}(a)=\mathrm{Read}\}.
\]

\begin{definition}[Directed resource-meeting witnesses]
For \(U,V\in\mathcal U\), define
\[
\mathcal M(U,V)
:=
\left\{
(a,b)\in\mathcal W(U)\times\mathcal R(V):
\operatorname{node}(a)=\operatorname{node}(b)
\right\}.
\]
Agents \(U\) and \(V\) have a directed resource meeting when
\[
\mathrm{RMeet}(U,V)
\iff
\mathcal M(U,V)\neq\varnothing.
\]
\end{definition}

\begin{definition}[Undirected A2A meeting]
Let
\[
\mathcal C_{\mathrm{A2A}}
\subseteq
\bigl\{\{U,V\}:U,V\in\mathcal U,\ U\neq V\bigr\}
\]
be the policy relation of agent pairs for which direct communication is
possible. Define
\[
\mathrm{A2AMeet}(U,V)
\iff
\{U,V\}\in\mathcal C_{\mathrm{A2A}}.
\]
Hence,
\[
\mathrm{A2AMeet}(U,V)
\iff
\mathrm{A2AMeet}(V,U).
\]
\end{definition}

\begin{definition}[Meeting]
For \(U,V\in\mathcal U\), define
\[
\mathrm{Meet}(U,V)
\iff
\mathrm{RMeet}(U,V)
\lor
\mathrm{A2AMeet}(U,V).
\]
\end{definition}

For action sets \(B,C\subseteq\mathcal A_{\mathrm{pol}}\), define
\[
\operatorname{Span}(B,C)
:=
\begin{cases}
\displaystyle
\max\{d_{\mathcal A}(a,b):a\in B,\ b\in C\},
&
B\times C\neq\varnothing,
\\[4pt]
0,
&
B\times C=\varnothing.
\end{cases}
\]

\begin{definition}[Influence potential]
For \(U,V\in\mathcal U\), define
\[
\Phi_{\mathrm{res}}(U\to V)
:=
\begin{cases}
\operatorname{Span}(\mathcal W(U),\mathcal R(V)),
&
\mathrm{RMeet}(U,V),
\\
0,
&
\text{otherwise,}
\end{cases}
\]
and
\[
\Phi_{\mathrm{A2A}}(U,V)
:=
\begin{cases}
\operatorname{Span}(\mathcal A(U),\mathcal A(V)),
&
\mathrm{A2AMeet}(U,V),
\\
0,
&
\text{otherwise.}
\end{cases}
\]
The meeting-conditioned influence potential is
\[
\Phi(U\to V)
:=
\max\{
\Phi_{\mathrm{res}}(U\to V),
\Phi_{\mathrm{A2A}}(U,V)
\}.
\]
\end{definition}

\begin{remark}[Role of the meeting condition]
A directed resource meeting establishes that \(U\) can write a resource node
that \(V\) can read. At that shared node, the corresponding write--read pair
has distance zero. This is intentional: the shared resource witnesses the
local interaction, while local least-privilege and mediation controls govern
that interaction itself.

An A2A meeting establishes that \(U\) and \(V\) can communicate directly.
It is treated conservatively as undirected because apparently one-way task
delegation may permit reverse influence through returned content, delegated
artifacts, tool outputs, or prompt-injection payloads.

Conditioned on either type of witness, \(\Phi(U\to V)\) measures the maximum
structural span of the relevant action sets. It is a configuration-risk
metric, not evidence of direct data flow, provenance, command authority, or
realized execution.
\end{remark}

\subsection{The Threshold $\theta$ as a Policy Parameter}

The threshold \(\theta\) is an enterprise policy parameter that determines
which meeting-conditioned influence potentials are sufficiently high impact
to be represented as arcs in the influence graph. It is selected by the
security policy owner; it is not defined to be the maximum blast radius in
the current catalogue or in any individual agent's assigned permissions.

The resource topology, path-accumulated labels, multiplier function \(F\),
and agent-independent policy catalogue provide calibration data for selecting
and reviewing \(\theta\). They do not mechanically determine its value.

Let
\[
\mathcal L_{\geq 3}
:=
\{L_{3C},L_{3I},L_4,L_{5C},L_{5I},L_6\},
\]
and let the material policy catalogue be
\[
\mathcal A_{\geq 3}
:=
\left\{
a\in\mathcal A_{\mathrm{pol}}:
J\!\bigl(\operatorname{node}(a)\bigr)\in\mathcal L_{\geq 3}
\right\}.
\]
Define the catalogue-radius profile
\[
\mathcal R_{\mathrm{cal}}
:=
\left\{
d_{\mathcal A}(a,b):
a,b\in\mathcal A_{\geq 3},\
\operatorname{node}(a)\neq\operatorname{node}(b)
\right\}.
\]
This finite set records the distinct nonzero blast-radius levels available
under the current enterprise hierarchy and materiality policy.

The materiality rule is used here to calibrate and review the policy
threshold \(\theta\). Unless stated otherwise, it does not restrict the
write and read action sets used in \(\Phi(U\to V)\): every authorized action
of an interacting pair may contribute to that pair's influence potential.

The security policy specifies
\[
\theta\in[0,+\infty)
\]
independently of the current agent assignments. In selecting \(\theta\), a
policy owner may use \(\mathcal R_{\mathrm{cal}}\) to identify meaningful
hierarchy and sensitivity boundaries, such as the smallest radius crossing a
Finance--Engineering boundary, a chosen upper quantile of catalogue radii,
or a documented organizational risk tolerance. The selected value, its
rationale, the materiality rule, and the catalogue version used for
calibration should be retained as policy evidence.

For a fixed \(\theta\), an influence potential is high impact when
\[
\Phi(U\to V)\geq\theta.
\]
Thus, lowering \(\theta\) weakly enlarges the influence graph, whereas
raising \(\theta\) weakly removes arcs. The threshold may be reconsidered
when the topology, labels, materiality policy, or enterprise risk tolerance
changes, but such a change is a policy decision rather than an automatic
consequence of adding an action to the catalogue.

The catalogue maximum remains useful as a calibration statistic:
\[
r_{\max}
:=
\begin{cases}
\max \mathcal R_{\mathrm{cal}},
&
\mathcal R_{\mathrm{cal}}\neq\varnothing,
\\[4pt]
0,
&
\mathcal R_{\mathrm{cal}}=\varnothing.
\end{cases}
\]
It bounds the largest blast radius represented by the material catalogue,
but it is not, in general, equal to \(\theta\).

\subsection{The Influence Graph and Effective Autonomy}

\begin{definition}[Influence graph]
For a policy threshold \(\theta\geq 0\), the influence graph is the directed
graph
\[
G(\theta)
=
\bigl(\mathcal U,E_\theta\bigr),
\]
where
\[
E_\theta
:=
\left\{
(U,V)\in\mathcal U\times\mathcal U:
U\neq V,\;
\mathrm{Meet}(U,V),\;
\Phi(U\to V)\geq\theta
\right\}.
\]
\end{definition}

An arc \(U\to V\) has two requirements. First,
\(\mathrm{Meet}(U,V)\) requires either a directed shared-resource witness
or an undirected A2A communication meeting. Second,
\(\Phi(U\to V)\geq\theta\) requires the relevant meeting-conditioned
influence potential to meet the policy threshold.

A resource-mediated arc retains a write--read witness
\((a^{\mathrm{meet}},b^{\mathrm{meet}})\in\mathcal M(U,V)\). An A2A-only
arc retains the communicating agent pair \(\{U,V\}\) as its meeting witness.

For every arc \((U,V)\in E_\theta\), the analysis retains at least one meeting 
witness.
For a resource-mediated arc, this is a write–read witness \((a^{\mathrm{meet}},b^{\mathrm{meet}})\in\mathcal M(U,V)\); 
for an A2A-only arc, it is the communicating pair \{U,V\}. When both conditions hold, both witness types may be retained.

\begin{definition}[Weak projection and candidate-coalition components]
Let \(H=(\mathcal U,E_H)\) be a directed graph on the agent set
\(\mathcal U\). Its \emph{weak projection} \(\overline H\) has an
undirected edge \(\{U,V\}\) whenever either \((U\to V)\in E_H\) or
\((V\to U)\in E_H\).

Each non-singleton weakly connected component of \(\overline H\) is called a
\emph{candidate-coalition component}. A coalition under consideration is
any set of at least two agents contained in one such component.
\end{definition}

Because the collusion predicate is witnessed by a pair of agents, testing
all ordered pairs within each candidate-coalition component is sufficient
to detect every collusion-positive subset; the implementation reports the
containing component rather than enumerating all such subsets.

\begin{definition}[Graph-relative reachability and effective autonomy]
Let \(H=(\mathcal U,E_H)\) be a directed graph on \(\mathcal U\). Write
\(U\leadsto_H V\) when there is a directed path of positive length from
\(U\) to \(V\) in \(H\). Write \(U\leadsto_H^{*}V\) when either \(U=V\) or
\(U\leadsto_H V\).

The \emph{effective autonomy set} of \(U\) relative to \(H\) is
\[
\mathcal E_H(U)
:=
\bigcup_{\substack{
V\in\mathcal U\\
U\leadsto_H^{*}V
}}
\mathcal A(V).
\]
For a coalition \(S\subseteq\mathcal U\), define
\[
\mathcal E_H(S)
:=
\bigcup_{U\in S}\mathcal E_H(U).
\]
\end{definition}

The reflexive closure in this definition ensures that an agent's effective
autonomy includes its own directly assigned actions as well as the actions
reachable through influence paths.

\begin{remark}[Conservative effective-autonomy abstraction]
For least-autonomy screening, directed reachability is treated as a
worst-case capability-exposure relation. Accordingly,
\(\mathcal E_H(U)\) includes all actions assigned to reachable agents.
This is a conservative policy abstraction; it does not assert command
authority, provenance, or realized execution along every directed path.
\end{remark}

\subsection{Computational Complexity}

Suppose there are $n$ agents and each agent has at most $m$ actions. After
standard preprocessing for constant-time LCA queries, direct computation of
all influence potentials examines at most
\[
n(n-1)m^2
\]
ordered write-to-read action pairs. It therefore takes
\[
O(n^2m^2)
\]
time in the worst case.

Constructing $G(\theta)$ and its weak projection takes
$O(n^2+|E_\theta|)$ time once the influence potentials are known.
Computing directed reachability by a graph search from every agent takes
\[
O\bigl(n(n+|E_\theta|)\bigr)
\]
time, and hence $O(n^3)$ in the worst case.

Testing membership in \(\mathcal C_{\mathrm{A2A}}\) for all ordered agent
pairs takes \(O(n^2)\) time with a relation matrix or hash-based
representation. This does not change the \(O(n^2m^2)\) worst-case bound.

A sharper bound requires additional structural assumptions or a specified
indexing scheme for action locations. In particular, no
$O(n^2m\Delta)$ bound is claimed here without an explicit aggregation
argument.

The bounds above cover influence-potential computation, graph construction,
and reachability. Full least-autonomy evaluation additionally depends on
the representation of effective-action sets, sensitivity domains, gate
relations, and collusion-witness enumeration.

\section{Collusion}
\label{sec:collusion}

The directed reachability relation $U\leadsto_H V$ records a
\emph{model-level potential influence path} from $U$ to $V$. It is not, by
itself, a provenance-preserving dataflow trace: 
successive arcs in a path may be witnessed by different resource write--read
pairs or A2A communication meetings.
The relation is therefore used as a conservative screening condition for possible
multi-agent composition.

\subsection{Role Scopes and Gate Relations}

\begin{definition}[Role scopes]
\label{def:rolescopes}
For $U\in\mathcal{U}$, define the \emph{approval scope}
\[
\mathcal{G}(U)
:=
\bigcup_{\substack{a\in\mathcal{A}(U)\\
\mathit{CF}(a)=GK}}
\mathrm{subtree}(\mathit{node}(a)),
\]
and the \emph{access scope}
\[
\mathcal{X}(U)
:=
\bigcup_{\substack{a\in\mathcal{A}(U)\\
\mathit{CF}(a)=Ac}}
\mathrm{subtree}(\mathit{node}(a)).
\]
\end{definition}

Scope overlap alone does not establish that an approval action authorizes
access to a resource. Let
\[
\rightsquigarrow_{\mathrm{gate}}
\ \subseteq\ N\times N
\]
be a security policy relation, where
\[
g\rightsquigarrow_{\mathrm{gate}}x
\]
means that approval authority associated with node $g$ can authorize,
enable, or release access to the resource location $x$.

\begin{definition}[Authorization relation]
\label{def:authorizationrelation}
For agents $U,V\in\mathcal{U}$, write
\[
\mathrm{Gates}(U,V)
\]
when there exist $g\in\mathcal{G}(U)$ and $x\in\mathcal{X}(V)$ such that
\[
g\rightsquigarrow_{\mathrm{gate}}x.
\]
\end{definition}

\subsection{Sensitivity Domains}

\begin{definition}[Sensitivity domain]
A \emph{sensitivity domain} is a policy-defined subset of enterprise
resource nodes. Let
\[
\mathfrak D
=
\left\{
\mathcal D_1,\ldots,\mathcal D_m
\right\}
\]
be a family of nonempty, pairwise disjoint subsets of \(N\). Each
\(\mathcal D_i\) represents a security-relevant operational domain, such as
a production environment, a financial-control boundary, or a regulated-data
environment.

The domain family need not be inferred from sensitivity labels or from the
topology of the enterprise tree. It is specified by the security engineer
and may reflect organizational, regulatory, deployment, or trust boundaries.

For each node \(v\in N\), define
\[
\operatorname{dom}(v)
=
\begin{cases}
i,
&
v\in\mathcal D_i
\text{ for some }
i\in\{1,\ldots,m\},
\\[2pt]
\bot,
&
v\notin\displaystyle\bigcup_{i=1}^{m}\mathcal D_i.
\end{cases}
\]

A symmetric, irreflexive policy relation
\[
\perp_{\mathrm{pol}}
\;\subseteq\;
\{1,\ldots,m\}\times\{1,\ldots,m\}
\]
specifies required separation between domains. For distinct
\(i,j\in\{1,\ldots,m\}\),
\[
i\perp_{\mathrm{pol}}j
\]
means that actions associated with resources in \(\mathcal D_i\) and
\(\mathcal D_j\) are intended to remain separated.
\end{definition}

\subsection{Collusion Predicate}

We consider a conservative threat model in which agents may coordinate,
share intermediate outputs, or exploit independently authorized capabilities
outside a single predeclared workflow. Accordingly, a risk witness need not
be confined to one workflow identifier or one shared business process.

The analysis does not claim that every witness corresponds to a currently
implemented end-to-end workflow. Rather, it identifies configurations in
which directed influence and separately assigned control or access
capabilities can be composed by cooperating, compromised, or improperly
orchestrated agents. Such a witness is sufficient to require policy review
under the least-autonomy criterion.

Let
\[
\mathrm{Loc}(B)
:=
\{
\operatorname{node}(a):a\in B
\}
\]
for every action set \(B\subseteq\mathcal A\).

\begin{definition}[Collusion predicate]
\label{def:collusion}
Let \(H=(\mathcal U,E_H)\) be a directed graph on the agent set
\(\mathcal U\). For any set \(S\subseteq\mathcal U\) with \(|S|\geq2\),
the predicate \(\mathrm{Collude}_H(S)\) holds if and only if there exist
distinct agents \(U,V\in S\) satisfying at least one of the following
conditions:
\begin{enumerate}[label=(\roman*)]
\item
\[
\mathrm{Gates}(U,V)
\quad\text{and}\quad
U\leadsto_H V
\]
\hfill (authorization composition);

\item
\[
\mathrm{Gates}(U,V)
\quad\text{and}\quad
V\leadsto_H U
\]
\hfill (decision manipulation);

\item
\[
U\leadsto_H V
\]
and there exist indices $(i,j)$ with
\[
i\perp_{\mathrm{pol}}j,
\qquad
\mathrm{Loc}(\mathcal E_H(U))
\cap\mathcal D_i\neq\varnothing,
\qquad
\mathrm{Loc}(\mathcal E_H(V))
\cap\mathcal D_j\neq\varnothing
\]
\hfill (cross-domain exposure composition).
\end{enumerate}
\end{definition}

Clause~(iii) is intentionally a conservative reachable-capability test.
Because effective autonomy is monotone along directed reachability, the two
separated-domain capabilities may be assigned to the same reachable agent.
Accordingly, clause~(iii) identifies cross-domain exposure within an
influence-connected configuration; unlike clauses~(i) and~(ii), it does not
require each witness agent to contribute a distinct domain capability.

\begin{remark}[Directionality]
Each clause requires a directed reachability witness in $H$.
Although $d(a,b)$ is symmetric, $\Phi(U\to V)$ need not equal
$\Phi(V\to U)$ because the maximization is over ordered write-to-read action
pairs.

A single agent $U$ for which $\mathrm{Gates}(U,U)$ holds is a
separation-of-duty concern. It is reported separately from collusion, which
concerns distinct agents.
\end{remark}

\begin{remark}[Relation to candidate-coalition components]
Every collusion witness in $H$ lies within one weakly connected component of
$\overline H$. Consequently, non-singleton weak components provide
a screening partition for collusion analysis. A large weak component does
not itself imply $\mathrm{Collude}$: the predicate additionally requires an
oriented influence witness and either an authorization or a cross-domain
condition.
\end{remark}

\begin{proposition}[Collusion monotonicity]
\label{prop:collusionmono}
Let $S,S'\subseteq\mathcal{U}$ with $|S|\geq2$ and $S\subseteq S'$.
Then:
\begin{enumerate}[label=(\roman*)]
\item If $\mathrm{Collude}_{G(\theta)}(S)$ and $\theta'\leq\theta$, then
\[
\mathrm{Collude}_{G(\theta')}(S).
\]

\item If $\mathrm{Collude}_{G(\theta)}(S)$, then
\[
\mathrm{Collude}_{G(\theta)}(S').
\]
\end{enumerate}
\end{proposition}

\begin{proof}
For (i), decreasing the threshold from $\theta$ to $\theta'$ can only add
arcs to $G(\theta)$, since
\[
\Phi(U\to V)\geq\theta
\quad\Longrightarrow\quad
\Phi(U\to V)\geq\theta'.
\]
Hence every directed-reachability witness remains valid, and each effective
autonomy set \(\mathcal E_{G(\theta)}(U)\) can only grow when \(\theta\) is lowered. The authorization and domain conditions are
independent of the threshold, so the original collusion witness persists.

For (ii), the same distinct witness agents $U,V\in S$ also belong to $S'$.
Thus the conditions witnessing $\mathrm{Collude}_{G(\theta)}(S)$ also witness
$\mathrm{Collude}_{G(\theta)}(S')$.
\end{proof}

\section{Least Autonomy as a Design Criterion}
\label{sec:design}

\subsection{The Three-Layer Theory}

The framework operates at three related layers:
\begin{enumerate}[label=(\arabic*)]
\item \textbf{Action layer.} The blast radius $d(a,b)$ measures the
structural security distance between actions.

\item \textbf{Agent layer.} The influence potential $\Phi(U\to V)$ and
directed graph $G(\theta)$ represent potential
influence among organizationally distant agents.

\item \textbf{System layer.} The collusion predicate identifies specific
multi-agent policy violations within the operational graph. 
\end{enumerate}

\paragraph{Two policy parameters.}
The threshold \(\theta\) and the individual-autonomy budget
\(\tau_{\mathrm{aut}}\) serve different purposes. The threshold \(\theta\)
selects the high-impact meeting-conditioned relations represented in the
influence graph \(G(\theta)\).
It does not constrain direct assignments of actions to
one agent. The separate policy parameter
\(\tau_{\mathrm{aut}}\) bounds the blast radius of the action pairs that one
agent may hold directly or acquire through autonomous influence.

\begin{definition}[Individual blast-radius bound]
Let \(H\) be an influence graph and let
\[
\tau_{\mathrm{aut}}\geq 0
\]
be an individual-autonomy budget. An agent \(U\in\mathcal U\) satisfies the
\emph{individual blast-radius bound} at budget
\(\tau_{\mathrm{aut}}\) when
\[
\forall a,b\in\mathcal E_H(U),
\qquad
a\neq b
\ \Longrightarrow\
d_{\mathcal A}(a,b)\leq\tau_{\mathrm{aut}}.
\]
Equivalently, whenever \(\mathcal E_H(U)\) contains at least two distinct
actions,
\[
\max\!\left\{
d_{\mathcal A}(a,b):
a,b\in\mathcal E_H(U),
\ a\neq b
\right\}
\leq
\tau_{\mathrm{aut}}.
\]
\end{definition}

\subsection{Least Autonomy Condition}

The operational graph records potential autonomous influence. A system should
permit such influence only where it is explicitly authorized. 

Let
\[
\mathsf{Allow}\subseteq\mathcal U\times\mathcal U
\]
be the policy relation of agent pairs for which autonomous influence is
permitted. 

\begin{definition}[Least autonomy condition]
A system satisfies \emph{least autonomy} at policy parameters
\[
\bigl(
\theta,
\tau_{\mathrm{aut}}
\bigr)
\]
if:
\begin{enumerate}[label=(\alph*)]
\item every agent satisfies the individual blast-radius bound in the
influence graph:
\[
\forall U\in\mathcal U,
\quad
\forall a,b\in\mathcal E_{G(\theta)}(U),
\quad
a\neq b
\ \Longrightarrow\
d_{\mathcal A}(a,b)\leq\tau_{\mathrm{aut}};
\]

\item every influence arc is explicitly authorized:
\[
\forall (U,V)\in E_\theta,
\qquad
(U,V)\in\mathsf{Allow};
\]

\item no set of distinct agents has a collusion witness in the 
graph:
\[
\mathrm{Collude}_{G(\theta)}(S)
=
\mathrm{false}
\qquad
\text{for every }
S\subseteq\mathcal U
\text{ with }
|S|\geq 2.
\]
\end{enumerate}
\end{definition}

Condition~(a) is logically independent of conditions~(b) and~(c). An
agent can violate the individual blast-radius bound through direct
permission assignments alone, even when it has no incoming or outgoing
influence arc. Conversely, authorization and the absence of a collusion
witness do not limit the concentration of harmful capability inside one
effective-autonomy set. The three conditions therefore control distinct
risks: individual concentration of authority, unauthorized autonomous
influence, and harmful multi-agent composition.

\subsection{Hierarchical Partitioning}

The ultrametric hierarchy supports a useful partitioning strategy, but it
does not by itself solve the least-autonomy problem. Security domains may be
organized as subtrees of $\mathcal{T}$, and the blast-radius metric helps
identify direct influence relations likely to cross domain boundaries.

However, direct partitioning must be followed by analysis of the 
graph, because influence may compose through intermediate agents. In
particular, requiring cross-domain direct influence potentials to remain below 
$\theta$ is helpful but does not by itself rule out multi-hop paths or
gatekeeping relationships. The least-autonomy condition therefore combines
domain design with explicit authorization, and
collusion analysis.

\subsection{Step-by-Step Design Procedure}
\label{sec:procedure}

Steps 1--4 establish the resource and policy model. The remaining steps are
repeated whenever permissions, resources, or workflow gates change.

\begin{enumerate}[label=\textbf{Step \arabic*.}, leftmargin=*, align=left]
\item \textbf{Construct the resource hierarchy and policy catalogue.}
Define the enterprise tree $\mathcal{T}$, identify resource locations,
and construct the agent-independent action catalogue
$\mathcal{A}_{\mathrm{pol}}$.

\item \textbf{Classify resource providers and assign local labels.}
Tag each provider as Standard, C, I, or CI, then assign the corresponding
local labels $\ell(v)$ to hierarchy nodes.

\item \textbf{Configure policy relations.}
Specify high-sensitivity nodes, sensitivity domains, the separation
relation \(\perp_{\mathrm{pol}}\), the gate relation
\(\rightsquigarrow_{\mathrm{gate}}\), the undirected A2A communication
relation \(\mathcal C_{\mathrm{A2A}}\), and the authorized agent-pair
relation \(\mathsf{Allow}\).

\item \textbf{Set and document the policy parameters.}
Select \(\theta\) from the enterprise's risk tolerance and intended
cross-scope control boundary. Use the material catalogue-radius profile
\(\mathcal R_{\mathrm{cal}}\) to calibrate and justify that selection.
Separately select \(\tau_{\mathrm{aut}}\) as the maximum authority span that
one agent may hold directly or acquire through autonomous influence.
Reassess both parameters when the policy model or risk tolerance changes.

\item \textbf{Compute meetings and influence potentials.}
Compute \(J(v)\) by a root-to-leaf traversal. For each ordered agent pair,
determine whether it has a directed resource meeting, an undirected A2A
meeting, or both. Use preprocessed LCA queries to compute the corresponding
meeting-conditioned influence potential \(\Phi(U\to V)\).

\item \textbf{Construct the influence graph.}
Build $G(\theta)$, its weak projection, and its non-singleton
candidate-coalition components.

\item \textbf{Evaluate least autonomy and iterate.}
Compute directed reachability and effective autonomy in
\(G(\theta)\). Check the individual blast-radius bound for
every agent, verify that every
arc \((U,V)\in E_\theta\) satisfies \((U,V)\in\mathsf{Allow}\), and evaluate the collusion predicate for each
candidate-coalition component. 
Where a condition fails, revise direct
permissions, the authorization relation, workflow design, hierarchy, or
policy thresholds, then recompute the graph.
\end{enumerate}

\subsection{Least Privilege vs.\ Least Autonomy}

\begin{table}[ht]
\centering
\caption{Least privilege and least autonomy.}
\label{tab:comparison}
\renewcommand{\arraystretch}{1.3}
\begin{tabular}{@{}p{3.2cm}p{4.4cm}p{5.2cm}@{}}
\toprule
\textbf{Dimension} & \textbf{Least Privilege} & \textbf{Least Autonomy} \\
\midrule
Primary unit of analysis
& Direct permission on a resource
& Potential influence and composition across agents \\

Typical scope
& Identity-to-resource access
& Gatekeeping and directed influence \\

Multi-agent reasoning
& Not explicit in the classical formulation
& Explicit through influence paths and collusion witnesses \\

Sensitivity representation
& Often policy-specific or binary
& Hierarchy plus lattice-valued sensitivity context \\

Control objective
& Minimize direct permissions
& Minimize unauthorized autonomous influence and harmful composition \\
\bottomrule
\end{tabular}
\end{table}

\section{Illustrative Example}
\label{sec:example}

This example uses only the action sets and role scopes already defined in the
main text. In particular, it does not introduce interface-specific action
subsets: directed meetings and influence potentials use
\(\mathcal W(U)\), \(\mathcal R(U)\), and \(\Phi(U\to V)\) exactly as in
Section~\ref{sec:influencegraph}; approval and access relations use
\(\mathcal G(U)\), \(\mathcal X(U)\), and
\(\rightsquigarrow_{\mathrm{gate}}\) exactly as in
Section~\ref{def:rolescopes}.

The example exercises only the resource-mediated branch of the generalized
meeting definition. No A2A communication pair is included in
\(\mathcal C_{\mathrm{A2A}}\).

\subsection{Hierarchy, Labels, and Policy Parameters}

The enterprise tree has root \(T\) at depth \(0\). At depth \(1\), it
separates into Operations and Corporate. At depth \(2\), Operations separates
into Finance, Engineering, and Compliance; Corporate contains Legal. At
depth \(3\), Finance contains Market Data, Treasury, and Controls;
Engineering contains Build Security and Production Security; Compliance
contains Audit; and Legal contains eDiscovery. The resources below are leaves
at depth \(4\).

All non-leaf nodes have local label \(L_0\). Therefore, for every resource
leaf \(r\) in Table~\ref{tab:example-labels},
\[
J(r)=\ell(r).
\]
The table uses only labels defined by the sensitivity lattice.

\begin{table*}[t]
\centering
\caption{Resource leaves, local labels, and accumulated labels.}
\label{tab:example-labels}

\setlength{\tabcolsep}{4pt}
\renewcommand{\arraystretch}{1.08}
\small

\begin{tabularx}{\textwidth}{
@{}
>{\raggedright\arraybackslash}p{0.22\textwidth}
>{\raggedright\arraybackslash}p{0.18\textwidth}
>{\raggedright\arraybackslash}X
c
c
@{}
}
\toprule
Division & Resource leaf & Purpose & \(\ell(r)\) & \(F(J(r))\) \\
\midrule
Finance / Market Data
& Rate Feed
& Threshold-only source
& \(L_{5I}\) & 1.75 \\

Finance / Treasury
& Wire Transfers
& Authorization-composition source
& \(L_{5I}\) & 1.75 \\

Finance / Market Data
& Market-Risk Report
& Decision-manipulation source
& \(L_{5C}\) & 1.75 \\

Finance / Treasury
& Settlement Note
& Threshold-and-budget source
& \(L_{5I}\) & 1.75 \\

Finance / Controls
& \(Q_{\theta}\) (Threshold Handoff Queue)
& Directed meeting
& \(L_4\) & 1.50 \\

Finance / Controls
& \(Q_{\theta\tau}\) (Both Handoff Queue)
& Directed meeting
& \(L_4\) & 1.50 \\

Finance / Controls
& \(Q_{\mathrm{auth}}\) (Signing Authorization Queue)
& Directed meeting and gate
& \(L_4\) & 1.50 \\

Finance / Controls
& \(Q_{\mathrm{review}}\) (Review Evidence Queue)
& Directed meeting
& \(L_4\) & 1.50 \\

Finance / Controls
& \(Q_{\mathrm{decision}}\) (Release Decision Queue)
& Reverse low-span meeting and gate
& \(L_4\) & 1.50 \\

Finance / Controls
& Requested Release
& Decision-gated Accessor target
& \(L_{5I}\) & 1.75 \\

Engineering / Build Security
& Signing Key
& Threshold-only target
& \(L_{5I}\) & 1.75 \\

Engineering / Build Security
& Build Attestation
& Threshold-and-budget target
& \(L_{5I}\) & 1.75 \\

Engineering / Build Security
& Release Manifest
& Authorization-composition target
& \(L_{5C}\) & 1.75 \\

Engineering / Production Security
& Deployment Policy
& Decision-manipulation target
& \(L_{5I}\) & 1.75 \\

Compliance / Audit
& Audit Evidence
& Isolated wide-authority action
& \(L_{5C}\) & 1.75 \\

Corporate / Legal / eDiscovery
& Legal Hold
& Isolated wide-authority action
& \(L_{5C}\) & 1.75 \\

Corporate / Legal / eDiscovery
& Corporate Audit Package
& Threshold-and-budget source
& \(L_{5C}\) & 1.75 \\

Compliance / Audit
& Compliance Bulletin
& Clean case
& \(L_2\) & 1.20 \\
\bottomrule
\end{tabularx}
\end{table*}

The hierarchy gives four relevant structural factors:
\[
\begin{aligned}
\operatorname{ultra}_{\mathrm{Controls}}
&=
2^{-(2\cdot3+1)}
=
\frac1{128},\\
\operatorname{ultra}_{\mathrm{Finance}}
&=
2^{-(2\cdot2+1)}
=
\frac1{32},\\
\operatorname{ultra}_{\mathrm{Fin\text{-}Eng}}
&=
2^{-(2\cdot1+1)}
=
\frac18,\\
\operatorname{ultra}_{\mathrm{Ops\text{-}Corp}}
&=
2^{-(2\cdot0+1)}
=
\frac12.
\end{aligned}
\]
The following reference radii are those used by the maximizing pairs below.
Each such pair contains at least one \(L_{5C}\) or \(L_{5I}\) endpoint, so
its pair multiplier is \(1.75\). Other action pairs in the same structural
band may have smaller radii when their endpoint labels are lower.
\[
\begin{aligned}
r_{\mathrm{Controls}}
&=
1.75\cdot\frac1{128}
=
0.013671875,\\
r_{\mathrm{Finance}}
&=
1.75\cdot\frac1{32}
=
0.0546875,\\
r_{\mathrm{Fin\text{-}Eng}}
&=
1.75\cdot\frac18
=
0.21875,\\
r_{\mathrm{Ops\text{-}Corp}}
&=
1.75\cdot\frac12
=
0.875.
\end{aligned}
\]

The example adopts
\[
\theta=0.21875,
\qquad
\tau_{\mathrm{aut}}=0.50000.
\]
The threshold is an externally selected policy boundary calibrated to the
Finance--Engineering radius. The larger Operations--Corporate radius
\(0.875\) remains catalogue calibration data and does not automatically
change \(\theta\).

\subsection{Agents, Actions, and Directed Meetings}

Table~\ref{tab:example-actions} specifies action control-flow type and
operation explicitly. A queue writer may be a GateKeeper action and its
consumer may be an Accessor read action; both actions remain members of the
ordinary write and read sets used by the influence-potential definition.

For brevity, the resource-provider and sensitivity fields of each action are
suppressed in the table; they are fixed by the policy catalogue for the
listed resource node. The calculations below use the associated node,
control-flow classification, and operation.

\begin{table*}[t]
\centering
\caption{Agent actions.  The notation \((GK,W)\) denotes a GateKeeper write
and \((Ac,R)\) an Accessor read.}
\label{tab:example-actions}
\renewcommand{\arraystretch}{1.12}
\small
\begin{tabular}{@{}p{0.16\linewidth}p{0.39\linewidth}p{0.35\linewidth}@{}}
\toprule
Agent & Write actions & Read actions \\
\midrule
\(U_{\tau}\)
&
Audit Evidence \((Ac,W)\)
&
Legal Hold \((Ac,R)\)
\\[3pt]
\(U_{\theta}^{s}\)
&
Rate Feed \((Ac,W)\);
\(Q_{\theta}\) \((GK,W)\)
&
-- \\
[3pt]
\(U_{\theta}^{t}\)
&
--
&
\(Q_{\theta}\) \((Ac,R)\);
Signing Key \((Ac,R)\)
\\[3pt]
\(U_{\theta\tau}^{s}\)
&
Settlement Note \((Ac,W)\);
Corporate Audit Package \((Ac,W)\);
\(Q_{\theta\tau}\) \((GK,W)\)
&
-- \\
[3pt]
\(U_{\theta\tau}^{t}\)
&
--
&
\(Q_{\theta\tau}\) \((Ac,R)\);
Build Attestation \((Ac,R)\)
\\[3pt]
\(U_{\mathrm{auth}}^{s}\)
&
Wire Transfers \((Ac,W)\);
\(Q_{\mathrm{auth}}\) \((GK,W)\)
&
-- \\
[3pt]
\(U_{\mathrm{auth}}^{t}\)
&
--
&
\(Q_{\mathrm{auth}}\) \((Ac,R)\);
Release Manifest \((Ac,R)\)
\\[3pt]
\(U_{\mathrm{req}}\)
&
Market-Risk Report \((Ac,W)\);
\(Q_{\mathrm{review}}\) \((Ac,W)\)
&
\(Q_{\mathrm{decision}}\) \((Ac,R)\);
Requested Release \((Ac,R)\)
\\[3pt]
\(U_{\mathrm{dec}}\)
&
\(Q_{\mathrm{decision}}\) \((GK,W)\)
&
\(Q_{\mathrm{review}}\) \((Ac,R)\);
Deployment Policy \((Ac,R)\)
\\[3pt]
\(U_0\)
&
--
&
Compliance Bulletin \((Ac,R)\)
\\
\bottomrule
\end{tabular}
\end{table*}

The nonempty directed meeting-witness sets are:
\[
\begin{array}{c|c}
\text{Ordered pair} & \text{Shared write--read node}\\
\hline
(U_{\theta}^{s},U_{\theta}^{t}) & Q_{\theta}\\
(U_{\theta\tau}^{s},U_{\theta\tau}^{t}) & Q_{\theta\tau}\\
(U_{\mathrm{auth}}^{s},U_{\mathrm{auth}}^{t}) & Q_{\mathrm{auth}}\\
(U_{\mathrm{req}},U_{\mathrm{dec}}) & Q_{\mathrm{review}}\\
(U_{\mathrm{dec}},U_{\mathrm{req}}) & Q_{\mathrm{decision}}
\end{array}
\]
No other ordered pair has a shared write--read node. In particular,
\(U_{\tau}\) has no directed meeting with any other agent.

\subsection{Influence Potentials and Graph Construction}

For the local reverse meeting,
\[
\begin{aligned}
\Phi(U_{\mathrm{dec}}\to U_{\mathrm{req}})
&=
d_{\mathcal A}
\bigl(
Q_{\mathrm{decision}},
\text{Requested Release}
\bigr)\\
&=
1.75\cdot\frac1{128}\\
&=
0.013671875
<
\theta.
\end{aligned}
\]
This meeting is genuine but does not produce an influence arc.

The remaining meeting-conditioned potentials are:
\[
\begin{aligned}
\Phi(U_{\theta}^{s}\to U_{\theta}^{t})
&=
d_{\mathcal A}
\bigl(
\text{Rate Feed},
\text{Signing Key}
\bigr)
=
0.21875,
\\
\Phi(U_{\theta\tau}^{s}\to U_{\theta\tau}^{t})
&=
d_{\mathcal A}
\bigl(
\text{Corporate Audit Package},
\text{Build Attestation}
\bigr)
=
0.875,
\\
\Phi(U_{\mathrm{auth}}^{s}\to U_{\mathrm{auth}}^{t})
&=
d_{\mathcal A}
\bigl(
\text{Wire Transfers},
\text{Release Manifest}
\bigr)
=
0.21875,
\\
\Phi(U_{\mathrm{req}}\to U_{\mathrm{dec}})
&=
d_{\mathcal A}
\bigl(
\text{Market-Risk Report},
\text{Deployment Policy}
\bigr)
=
0.21875.
\end{aligned}
\]
The shared queue is the directed meeting witness in each case; the displayed
remote action pair attains the maximum and measures operational span. It is
not asserted to be a direct data-flow channel.

Therefore,
\[
E_\theta
=
\left\{
(U_{\theta}^{s},U_{\theta}^{t}),
(U_{\theta\tau}^{s},U_{\theta\tau}^{t}),
(U_{\mathrm{auth}}^{s},U_{\mathrm{auth}}^{t}),
(U_{\mathrm{req}},U_{\mathrm{dec}})
\right\}.
\]
Every listed arc is admitted through a directed resource meeting and
satisfies \(\Phi(U\to V)\geq\theta\).

\subsection{Individual Blast Radius, Authorization, and Collusion}

For concise reporting, define the direct and effective maximum radii by
\[
\beta_{\mathrm{dir}}(U)
:=
\begin{cases}
\displaystyle
\max\left\{
d_{\mathcal A}(a,b):
a,b\in\mathcal A(U),\
a\neq b
\right\},
&
|\mathcal A(U)|\geq2,
\\[6pt]
0,
&
|\mathcal A(U)|<2,
\end{cases}
\]
and
\[
\beta_H(U)
:=
\begin{cases}
\displaystyle
\max\left\{
d_{\mathcal A}(a,b):
a,b\in\mathcal E_H(U),\
a\neq b
\right\},
&
|\mathcal E_H(U)|\geq2,
\\[6pt]
0,
&
|\mathcal E_H(U)|<2.
\end{cases}
\]
These are reporting shorthands for the individual blast-radius bound in
Section~\ref{sec:design}; they introduce no new action or scope types.

All influence arcs are authorized in this baseline:
\[
\mathsf{Allow}=E_\theta.
\]
This choice isolates the individual-radius and collusion conditions; the
example does not include an unauthorized-influence case.

The following calculations distinguish the six cases:
\[
\begin{array}{c|c|c}
\text{Agent} & \beta_{\mathrm{dir}}(U) & \beta_H(U)\\
\hline
U_{\tau} & 0.875 & 0.875\\
U_{\theta}^{s} & 0.0546875 & 0.21875\\
U_{\theta}^{t} & 0.21875 & 0.21875\\
U_{\theta\tau}^{s} & 0.875 & 0.875\\
U_{\theta\tau}^{t} & 0.21875 & 0.21875\\
U_{\mathrm{auth}}^{s} & 0.0546875 & 0.21875\\
U_{\mathrm{auth}}^{t} & 0.21875 & 0.21875\\
U_{\mathrm{req}} & 0.0546875 & 0.21875\\
U_{\mathrm{dec}} & 0.21875 & 0.21875\\
U_0 & 0 & 0
\end{array}
\]
Thus, every agent except \(U_{\tau}\) and
\(U_{\theta\tau}^{s}\) has individual direct radius below \(0.25\), and
every agent in the threshold-only, authorization-composition, and
decision-manipulation components has effective radius
\(0.21875<\tau_{\mathrm{aut}}\).

For collusion analysis, specify exactly two gate relations:
\[
Q_{\mathrm{auth}}
\rightsquigarrow_{\mathrm{gate}}
\text{Release Manifest},
\qquad
Q_{\mathrm{decision}}
\rightsquigarrow_{\mathrm{gate}}
\text{Requested Release}.
\]
All other gate-relation pairs are absent. Hence,
\[
\begin{aligned}
\mathrm{Gates}
(U_{\mathrm{auth}}^{s},U_{\mathrm{auth}}^{t})
&=\mathrm{true},\\
\mathrm{Gates}
(U_{\mathrm{dec}},U_{\mathrm{req}})
&=\mathrm{true}.
\end{aligned}
\]
Set
\[
\perp_{\mathrm{pol}}=\varnothing
\]
to isolate the authorization and decision clauses of the collusion predicate.

The following are local agent-level or component-level diagnostic outcomes.
They do not imply that the baseline configuration as a whole satisfies least
autonomy.

\begin{enumerate}
\item \textbf{Individual-budget breach only.}
\[
\beta_H(U_{\tau})=0.875>\tau_{\mathrm{aut}}.
\]
Agent \(U_{\tau}\) has no directed meeting, no influence arc, and no
collusion witness.

\item \textbf{Material influence only.}
The pair
\[
(U_{\theta}^{s},U_{\theta}^{t})
\]
has
\[
\Phi(U_{\theta}^{s}\to U_{\theta}^{t})
=
\theta
=
0.21875,
\]
while both agents satisfy the individual budget and have no collusion
witness.

\item \textbf{Both material influence and individual-budget breach.}
\[
\Phi(U_{\theta\tau}^{s}\to U_{\theta\tau}^{t})
=
0.875
\geq\theta,
\qquad
\beta_H(U_{\theta\tau}^{s})
=
0.875
>
\tau_{\mathrm{aut}}.
\]
No gate relation is defined for this component.

\item \textbf{Authorization-composition collusion.}
\[
\mathrm{Gates}
(U_{\mathrm{auth}}^{s},U_{\mathrm{auth}}^{t})
\quad\text{and}\quad
U_{\mathrm{auth}}^{s}
\leadsto_H
U_{\mathrm{auth}}^{t}.
\]
Therefore,
\[
\mathrm{Collude}_H
\bigl(
\{U_{\mathrm{auth}}^{s},U_{\mathrm{auth}}^{t}\}
\bigr)
=
\mathrm{true}.
\]
Both agents remain \(\tau_{\mathrm{aut}}\)-compliant.

\item \textbf{Decision-manipulation collusion.}
\[
\mathrm{Gates}
(U_{\mathrm{dec}},U_{\mathrm{req}})
\quad\text{and}\quad
U_{\mathrm{req}}
\leadsto_H
U_{\mathrm{dec}}.
\]
Therefore,
\[
\mathrm{Collude}_H
\bigl(
\{U_{\mathrm{req}},U_{\mathrm{dec}}\}
\bigr)
=
\mathrm{true}.
\]
The gate direction is opposite to the influence direction, as required by
clause~(ii) of the collusion predicate. Both agents remain
\(\tau_{\mathrm{aut}}\)-compliant.

\item \textbf{Locally clean case.}
Agent \(U_0\) has one standard Accessor action, no meeting witness, no
incident influence arc, no gate relation, and
\[
\beta_H(U_0)=0.
\]
It satisfies the individual blast-radius bound and has no incident influence
arc or local collusion witness. The baseline configuration as a whole still
fails least autonomy because other components are collusion-positive.
\end{enumerate}

The \(\theta\) threshold is an edge-admission policy, not a standalone
per-agent budget. Consequently, the two collusion cases must be incident to
\(\theta\)-qualified arcs in order to appear in \(H\) and be evaluated by
the current collusion predicate. Their intended interpretation is
\emph{individually narrow and \(\tau_{\mathrm{aut}}\)-compliant, but
collusion-positive}.

\section{Limitations and Future Work}
\label{sec:limitations}

\paragraph{Directed and semantic influence.}
The operational graph records model-level potential influence, not a
provenance-preserving account of dataflow. A directed path can be witnessed
by different write-to-read action pairs on its successive arcs. 

\paragraph{Static workflows and action sequencing.}
The model does not represent time, workflow state, conditional approvals,
revocation delays, or the semantics of long action sequences. Consequently,
it may over- or under-approximate risk in systems where these features
materially affect whether an influence path can be exercised.

\paragraph{Operation-sensitive risk.}
The current blast-radius metric measures structural separation between
resource nodes and does not directly distinguish operation types or
control-flow classifications. Consequently, a read action and a destructive
write or delete action associated with the same node have zero structural
distance. Such actions may still be significant under direct-permission,
gate, or collusion checks, but the model does not currently assign them
different local blast radii. Extending the framework with an
operation-severity taxonomy is left to future work.

\paragraph{Policy-model specification.}
The operational analysis depends on an accurately specified resource
hierarchy, action catalogue, label assignment, gate relation, domain
separation relation, and authorization relation.
Provider-level classification may be a useful practical approximation, but it can obscure
resource-level variation. Methods for eliciting, validating, and maintaining
these policy inputs are outside the present scope.

\paragraph{Availability and safety.}
The current label structure addresses confidentiality and integrity only.
Availability, physical safety, financial loss, and other risk dimensions
would require additional labels, policy relations, and validation criteria.

\paragraph{Coalition-level quantitative bounds.}
The framework identifies collusion witnesses but does not provide
quantitative bounds on coalition effective autonomy in terms of individual
blast radii. Establishing such bounds, if possible under suitable structural
assumptions, remains future work.

\paragraph{Direct agent-to-agent communication.}
The model represents direct A2A communication through a binary undirected
meeting relation. This is intentionally conservative: a configured
communication path is treated as permitting reciprocal influence, even when
the operational workflow appears one-way. The model does not yet distinguish
channel authentication strength, message provenance, protocol semantics,
content filtering, session state, or runtime enforcement.

\section{Conclusion}
\label{sec:conclusion}

Least privilege remains necessary for agentic AI systems, but it does not by
itself characterize the authority exposure that can arise when permissions
are distributed across interacting agents. This paper introduced least
autonomy as a complementary design criterion for that setting.

The theory has three layers. First, the compositional blast radius
\(d(a,b)\) provides a pseudoultrametric over actions that combines
enterprise hierarchy with accumulated sensitivity. Second, the
meeting-conditioned influence potential \(\Phi(U\to V)\) and influence graph
\(G(\theta)\) identify high-impact resource-mediated and A2A-mediated
interactions.
Third, effective autonomy and the collusion predicate screen for
authorization composition, decision manipulation, and cross-domain
capability composition under a conservative threat model.

The framework is deliberately a static policy-analysis abstraction. A
positive influence or collusion result identifies an auditable configuration
that warrants review; it does not by itself establish provenance, command
authority, malicious intent, or a realized end-to-end execution trace.
Likewise, the framework requires policy inputs beyond provider
classification, including the resource hierarchy, action catalogue,
thresholds, authorization and gate relations, and separation policy.

The intended use is therefore to complement existing access-control and
agent-runtime safeguards: it helps security engineers identify where
otherwise individually plausible permissions create excessive autonomous
reach or capability-composition risk. Future work should evaluate the
calibration of \(\theta\) and \(\tau_{\mathrm{aut}}\), the precision of the
screening results, and the integration of workflow and runtime evidence.

\newpage

\appendix
\section{The Compositional Blast Radius Is Pseudoultrametric}
\label{app:ultrametric}

Let \(N\) denote the node set of the rooted enterprise tree. For
\(x,y\in N\), define the underlying hierarchical distance
\[
u(x,y)
:=
\begin{cases}
0, & x=y,\\[2pt]
2^{-\left(2D\!\bigl(\operatorname{lca}(x,y)\bigr)+1\right)},
& x\neq y.
\end{cases}
\]

Let \(J:N\to\mathcal L\) assign a path-accumulated security label to each
node, and let
\[
F:\mathcal L\to(0,\infty)
\]
be the corresponding multiplier function. Assume that \(\mathcal L\) is
finite, and define
\[
F_{\min}
:=
\min_{\ell\in\mathcal L}F(\ell),
\qquad
F_{\max}
:=
\max_{\ell\in\mathcal L}F(\ell).
\]
Assume that
\[
\frac{F_{\max}}{F_{\min}}<4.
\]

For each node \(x\in N\), write
\[
m(x):=F(J(x)).
\]
Define the node-level compositional blast radius by
\[
d_N(x,y)
:=
\max\!\left\{
m(x),
m(y)
\right\}
u(x,y),
\qquad
x,y\in N.
\]

\begin{proposition}
\label{prop:node-ultrametric}
The function \(d_N\) is an ultrametric on \(N\).
\end{proposition}

\begin{proof}
The function \(d_N\) is symmetric because both
\[
\max\!\left\{
m(x),
m(y)
\right\}
\]
and \(u(x,y)\) are symmetric in \(x\) and \(y\). Moreover,
\[
d_N(x,x)=0
\]
for every \(x\in N\), since \(u(x,x)=0\). If \(x\neq y\), then
\(u(x,y)>0\) and
\[
\max\!\left\{
m(x),
m(y)
\right\}>0,
\]
so \(d_N(x,y)>0\).

It remains to prove the strong triangle inequality. Let
\(x,y,z\in N\). If two of the nodes coincide, the claim follows immediately
from symmetry and \(d_N(w,w)=0\) for every \(w\in N\). Hence suppose that
\(x\), \(y\), and \(z\) are distinct.

For three distinct nodes in a rooted tree, either all three pairwise least
common ancestors coincide, or exactly two pairwise least common ancestors
coincide and lie strictly above the third.

First suppose that
\[
\operatorname{lca}(x,y)
=\operatorname{lca}(x,z)
=\operatorname{lca}(y,z)=
p.
\]
Then define
\[
U:= 2^{-\left(2D(p)+1\right)}
\]
Since all three pairs have least common ancestor \(p\),
\[
u(x,y)=u(x,z)=u(y,z)=U.
\]
Let
\[
M
:=
\max\!\left\{
m(x),
m(y),
m(z)
\right\}.
\]
We have
\[
d_N(x,y)=
\max\!\left\{
m(x),
m(y)
\right\} \cdot U
\leq
M\cdot U.
\]
On the other hand,
\[
\begin{aligned}
\max\!\left\{
d_N(x,z),
d_N(y,z)
\right\}
&=
\max\!\left\{
\max\!\left\{
m(x),
m(z)
\right\}\cdot U,
\max\!\left\{
m(y),
m(z)
\right\}\cdot U
\right\}\\
&=
M\cdot U.
\end{aligned}
\]
Therefore,
\[
d_N(x,y)
\leq
\max\!\left\{
d_N(x,z),
d_N(y,z)
\right\}.
\]

Otherwise, after relabelling the nodes if necessary, there exist
\(p,q\in N\) such that
\[
\operatorname{lca}(x,z)
=\operatorname{lca}(y,z)=
p,
\qquad
\operatorname{lca}(x,y)
=
q,
\]
where \(q\) is a strict descendant of \(p\). Let
\[
U
:=
u(x,z)=
u(y,z)=
2^{-\left(2D(p)+1\right)}.
\]
Since \(q\) is a strict descendant of \(p\),
\[
D(q)\geq D(p)+1,
\]
and hence
\[
u(x,y)=
2^{-\left(2D(q)+1\right)}
\leq
\frac{U}{4}.
\]

Let
\[
M
:=
\max\!\left\{
m(x),
m(y),
m(z)
\right\}.
\]
For the two longer pairs,
\[
\begin{aligned}
\max\!\left\{
d_N(x,z),
d_N(y,z)
\right\}
&=
\max\!\left\{
\max\!\left\{
m(x),
m(z)
\right\}\cdot U,
\max\!\left\{
m(y),
m(z)
\right\}\cdot U
\right\}\\
&=
M\cdot U.
\end{aligned}
\]
For the remaining pair,
\[
\begin{aligned}
d_N(x,y)
&=
\max\!\left\{
m(x),
m(y)
\right\}
\cdot u(x,y)\\
&\leq
F_{\max}\frac{U}{4}\\
&<
F_{\min}U\\
&\leq
MU\\
&=
\max\!\left\{
d_N(x,z),
d_N(y,z)
\right\}.
\end{aligned}
\]
Thus,
\[
d_N(x,y)
\leq
\max\!\left\{
d_N(x,z),
d_N(y,z)
\right\}.
\]

The other orderings follow by relabelling \(x\), \(y\), and \(z\).
Therefore \(d_N\) satisfies the strong triangle inequality and is an
ultrametric on \(N\).
\end{proof}

Let
\[
\operatorname{node}:\mathcal A\to N
\]
map every action to its associated resource node. For actions
\(a,b\in\mathcal A\), define
\[
d_{\mathcal A}(a,b)
:=
d_N\!\bigl(
\operatorname{node}(a),
\operatorname{node}(b)
\bigr).
\]
Equivalently, if
\[
u_a:=\operatorname{node}(a),
\qquad
u_b:=\operatorname{node}(b),
\]
then
\[
d_{\mathcal A}(a,b)=
\max\!\left\{
F(J(u_a)),
F(J(u_b))
\right\}
\operatorname{ultra}(u_a,u_b).
\]

\begin{corollary}
\label{cor:action-pseudoultrametric}
The function \(d_{\mathcal A}\) is a pseudoultrametric on \(\mathcal A\).
\end{corollary}

\begin{proof}
Symmetry and the strong triangle inequality follow directly from
Proposition~\ref{prop:node-ultrametric}. Further,
\[
d_{\mathcal A}(a,a)=0
\]
for every \(a\in\mathcal A\). However, distinct actions \(a\neq b\) may be
associated with the same resource node, in which case
\[
\operatorname{node}(a)=\operatorname{node}(b)
\]
and therefore
\[
d_{\mathcal A}(a,b)=0.
\]
Thus \(d_{\mathcal A}\) need not separate distinct actions and is, in
general, a pseudoultrametric rather than an ultrametric.
\end{proof}

\section{Pseudocode for influence graph and least-autonomy evaluation}
\label{app:pseudocode}

\begingroup
\captionsetup{type=algorithm}
\label{alg:phi}
\begin{algorithmic}[1]

\Require Tree \(\mathcal T\); agents \(\mathcal U\); action sets; undirected A2A communication relation \(\mathcal C_{\mathrm{A2A}}\);
\(\{\mathcal A(U)\}_{U\in\mathcal U}\); policy thresholds
\(\theta\) and \(\tau_{\mathrm{aut}}\); gate relation
\(\rightsquigarrow_{\mathrm{gate}}\); sensitivity domains
\(\{\mathcal D_i\}\); separation relation
\(\perp_{\mathrm{pol}}\); authorization relation \(\mathsf{Allow}\)

\Ensure Influence graph \(G(\theta)\); weak components of \(\overline{G}(\theta)\);
influence potentials \(\Phi\); effective-autonomy sets
\(\{\mathcal E_H(U)\}\); effective blast radii \(\{\beta_H(U)\}\);
meeting witnesses; collusion witnesses; a policy-violation report;
and \(\mathrm{LeastAutonomy}\)

\Statex \textbf{Phase 1: Precomputation}

\State Compute \(J(v)\) and \(D(v)\) for all \(v\in N\)
\State Preprocess \(\mathcal T\) for LCA queries

\For{each \(U\in\mathcal U\)}
  \State Compute \(\mathcal W(U)\), \(\mathcal R(U)\),
  \(\mathcal G(U)\), and \(\mathcal X(U)\)
\EndFor

\State Define \(\mathrm{Gates}(U,V)\) from
\(\rightsquigarrow_{\mathrm{gate}}\), \(\mathcal G(U)\), and
\(\mathcal X(V)\)

\Statex \textbf{Phase 2: Directed meetings, influence potentials, and graph}

\State Initialize \(E_\theta\gets\varnothing\)
\State Initialize \(\mathsf{MeetingWitnesses}\gets\varnothing\)
\State Initialize \(\Phi(U\to V)\gets0\) for all
\(U,V\in\mathcal U\) with \(U\neq V\)

\For{each ordered pair \((U,V)\in\mathcal U\times\mathcal U\) with \(U\neq V\)}

  \State \(\mathcal M(U,V)\gets
  \{(a,b)\in\mathcal W(U)\times\mathcal R(V):
  \operatorname{node}(a)=\operatorname{node}(b)\}\)

  \State \(\mathrm{RMeet}(U,V)\gets
  \bigl(\mathcal M(U,V)\neq\varnothing\bigr)\)

  \State \(\mathrm{A2AMeet}(U,V)\gets
  \bigl(\{U,V\}\in\mathcal C_{\mathrm{A2A}}\bigr)\)

  \State \(\mathrm{Meet}(U,V)\gets
  \mathrm{RMeet}(U,V)\lor\mathrm{A2AMeet}(U,V)\)

  \If{\textbf{not} \(\mathrm{Meet}(U,V)\)}
    \State \textbf{continue}
  \EndIf

  \State \(\Phi_{\mathrm{res}}(U\to V)\gets0\)
  \If{\(\mathrm{RMeet}(U,V)\)}
    \State \((a^{\mathrm{meet}}_{UV},b^{\mathrm{meet}}_{UV})
    \gets\) any element of \(\mathcal M(U,V)\)
    \State \(r_{UV}\gets
    \operatorname{node}(a^{\mathrm{meet}}_{UV})\)
    \State Store
    \((\mathsf{Resource},U,V,a^{\mathrm{meet}}_{UV},
    b^{\mathrm{meet}}_{UV},r_{UV})\)
    in \(\mathsf{MeetingWitnesses}\)
    \State \(\Phi_{\mathrm{res}}(U\to V)\gets
    \operatorname{Span}(\mathcal W(U),\mathcal R(V))\)
  \EndIf

  \State \(\Phi_{\mathrm{A2A}}(U,V)\gets0\)
  \If{\(\mathrm{A2AMeet}(U,V)\)}
    \State Store \((\mathsf{A2A},\{U,V\})\)
    in \(\mathsf{MeetingWitnesses}\)
    \State \(\Phi_{\mathrm{A2A}}(U,V)\gets
    \operatorname{Span}(\mathcal A(U),\mathcal A(V))\)
  \EndIf

  \State \(\Phi(U\to V)\gets
  \max\{\Phi_{\mathrm{res}}(U\to V),
  \Phi_{\mathrm{A2A}}(U,V)\}\)

  \If{\(\Phi(U\to V)\geq\theta\)}
    \State Add \((U\to V)\) to \(E_\theta\)
  \EndIf

\EndFor

\State \(G(\theta)\gets(\mathcal U,E_\theta)\)

\Statex \textbf{Phase 3: Effective autonomy and policy conditions}

\State \(H\gets G(\theta)\)
\State Initialize \(\mathsf{Violations}\gets\varnothing\)
\State Compute directed reachability \(\leadsto_H\) in \(H\)

\For{each \(U\in\mathcal U\)}

  \State \(\mathcal E_H(U)\gets
  \bigcup\{\mathcal A(V):V\in\mathcal U,\ U\leadsto_H^{*}V\}\)

  \State \(\beta_H(U)\gets0\)
  \State \((a_U^\star,b_U^\star)\gets\bot\)

  \For{each unordered pair \(\{a,b\}\subseteq\mathcal E_H(U)\) with \(a\neq b\)}
    \State \(\rho\gets d_{\mathcal A}(a,b)\)
    \If{\(\rho>\beta_H(U)\)}
      \State \(\beta_H(U)\gets\rho\)
      \State \((a_U^\star,b_U^\star)\gets(a,b)\)
    \EndIf
  \EndFor

  \If{\(\beta_H(U)>\tau_{\mathrm{aut}}\)}
    \State Record
    \((\mathsf{BlastRadius},U,a_U^\star,b_U^\star,
    \beta_H(U),\tau_{\mathrm{aut}})\)
    in \(\mathsf{Violations}\)
  \EndIf

\EndFor

\For{each \((U,V)\in E_\theta\)}
  \If{\((U,V)\notin\mathsf{Allow}\)}
    \State Record
    \((\mathsf{UnauthorizedInfluence},U,V,\Phi(U\to V))\)
    in \(\mathsf{Violations}\)
  \EndIf
\EndFor

\State Construct \(\overline H\) and its non-singleton weak components
\(\{C_1,\ldots,C_k\}\)

\Statex \textbf{Phase 4: Collusion evaluation}

\State Initialize \(\mathsf{CollusionWitnesses}\gets\varnothing\)

\For{each component \(C\in\{C_1,\ldots,C_k\}\)}

  \State \(\mathsf{hasWitness}(C)\gets\mathrm{false}\)

  \For{each ordered pair \((U,V)\in C\times C\) with \(U\neq V\)}

    \If{\(\mathrm{Gates}(U,V)\) and \(U\leadsto_H V\)}
      \State Record an authorization-composition witness for \(C\)
      \State Store the witness in \(\mathsf{CollusionWitnesses}\)
      \State \(\mathsf{hasWitness}(C)\gets\mathrm{true}\)
    \EndIf

    \If{\(\mathrm{Gates}(U,V)\) and \(V\leadsto_H U\)}
      \State Record a decision-manipulation witness for \(C\)
      \State Store the witness in \(\mathsf{CollusionWitnesses}\)
      \State \(\mathsf{hasWitness}(C)\gets\mathrm{true}\)
    \EndIf

    \For{each ordered domain pair \((i,j)\) with \(i\perp_{\mathrm{pol}}j\)}
      \If{\(U\leadsto_H V\) and
      \(\operatorname{Loc}(\mathcal E_H(U))\cap\mathcal D_i\neq\varnothing\)
      and
      \(\operatorname{Loc}(\mathcal E_H(V))\cap\mathcal D_j\neq\varnothing\)}
        \State Record a cross-domain-composition witness for \(C\)
        \State Store the witness in \(\mathsf{CollusionWitnesses}\)
        \State \(\mathsf{hasWitness}(C)\gets\mathrm{true}\)
      \EndIf
    \EndFor

  \EndFor

  \If{\(\mathsf{hasWitness}(C)\)}
    \State Flag \(C\) as collusion-positive
    \State Record \((\mathsf{Collusion},C)\) in \(\mathsf{Violations}\)
  \EndIf

\EndFor

\Statex \textbf{Phase 5: Least-autonomy decision}

\If{\(\mathsf{Violations}=\varnothing\)}
  \State \(\mathrm{LeastAutonomy}\gets\mathrm{true}\)
\Else
  \State \(\mathrm{LeastAutonomy}\gets\mathrm{false}\)
\EndIf

\State \Return \(G(\theta)\), \(\Phi\),
\(\{\mathcal E_H(U)\}_{U\in\mathcal U}\),
\(\{\beta_H(U)\}_{U\in\mathcal U}\),
\(\mathsf{MeetingWitnesses}\),
\(\mathsf{CollusionWitnesses}\),
\(\mathsf{Violations}\), and \(\mathrm{LeastAutonomy}\)

\end{algorithmic}
\endgroup

\end{document}